\documentclass{article}

\usepackage{amsmath, amssymb, amsthm}

\newtheorem{corollary}{Corollary}[section]
\newtheorem{theorem}{Theorem}[section]
\newtheorem{lemma}{Lemma}[section]

\newcommand{\argmax}{\mathop{\rm argmax}}

\def\EE{\mathbb{E}}

\def\Qexpl{{(D,\pi_t(\cdot|x))_{t=1}^{T}}}
\def\Qpexpl{(D, a \sim \pi(\cdot|x))}

\usepackage{graphicx} 
\usepackage{subfigure} 

\usepackage{natbib}

\usepackage{algorithm}
\usepackage{algorithmic}

\newcommand{\chomp}[1]{}

\usepackage{hyperref}

\usepackage[accepted]{icml2010} 

\icmltitlerunning{Learning from Logged Implicit Exploration Data}

\begin{document} 

\twocolumn[
\icmltitle{Learning from Logged Implicit Exploration Data}

\icmlauthor{Alex Strehl}{astrehl@gmail.com}
\icmladdress{Facebook}
\icmlauthor{John Langford}{jl@yahoo-inc.com}
\icmladdress{Yahoo!}
\icmlauthor{Sham Kakade}{shamkakade@gmail.com}
\icmladdress{University of Pennsylvania}

\icmlkeywords{ranking, internet advertising, supervised learning, reinforcement learning, exploration}

\vskip 0.3in
]


\begin{abstract}
We provide a sound and consistent foundation for the use of
\emph{nonrandom} exploration data in ``contextual bandit'' or
``partially labeled'' settings where only the value of a chosen action
is learned. 

The primary challenge in a variety of settings is that the exploration
policy, in which ``offline'' data is logged, is not explicitly known.
Prior solutions here require either control of the actions
during the learning process, recorded random exploration, or actions
chosen obliviously in a repeated manner.  The techniques reported here
lift these restrictions, allowing the learning of a policy for
choosing actions given features from historical data where no
randomization occurred or was logged.

We empirically verify our solution on a reasonably sized set of
real-world data obtained from an online advertising company.
\end{abstract}

\section{The Problem}
\label{prob}
Consider the advertisement display problem, where a search engine
company chooses an ad to display which is intended to interest the
user.  Revenue is typically provided to the search engine from the
advertiser only when the user clicks on the displayed ad.  This
problem is of intrinsic economic interest, resulting in a substantial
fraction of income for several well known companies such as Google,
Yahoo!, and Facebook.  Furthermore, existing trends imply this
problem is of growing importance.

Before discussing the approach we propose, it's important to formalize
and generalize the problem, and then consider why more conventional
approaches can fail.

\subsection*{The warm start problem for contextual exploration}
\label{s:bandit}

Let $\mathcal{X}$ be an arbitrary input space, and $\mathcal{A} =
\{1,\cdots,k\}$ be a set of actions. An instance of the {\it contextual bandit problem} 
is specified by a distribution $D$ over
tuples $\left(x,\vec{r}\right)$ where $x \in \mathcal{X}$ is an input
and $\vec{r} \in [0,1]^k$ is a vector of rewards
\cite{langford08b}. Events occur on a round by round basis where on
each round $t$:
\begin{enumerate}
\item The world draws $(x,\vec{r})\sim D$ and announces $x$.
\item The algorithm chooses an action $a \in\mathcal{A}$, possibly
as a function of $x$ and historical information.
\item The world announces the reward $r_a$ of action $a$.
\end{enumerate}
It is critical to understand that this is not a standard supervised
learning problem, because the reward of other actions $a' \neq a$ is
not revealed.

The standard goal in this setting is to maximize the sum of rewards
$r_a$ over the rounds of interaction.  In order to do this well,
it is essential to use previously recorded events to form a good
policy on the first round of interaction.  This is known as the ``warm
start'' problem, and is the subject of this paper.  Formally, given
a dataset of the form $S = (x,a,r_a)^*$ generated by the interaction of
an uncontrolled logging policy, we want to construct a policy $h$
maximizing (or approximately maximizing)
\[
V^h := E_{(x,\vec{r})\sim D}[ r_{h(x)} ].
\]

\subsection*{Approaches that fail}

There are several approaches that may appear to solve this problem, but turn out to be inadequate: 
\begin{enumerate}
\item {\it Supervised learning}.  We could learn a regressor $s: X
  \times A \rightarrow [0,1]$ which is trained to predict the reward,
  on observed events conditioned on the action $a$ and other
  information $x$.  From this regressor, a policy is derived according
  to $h(x) = \argmax_{a\in A} s(x,a)$.  A flaw of this approach is
  that the $\argmax$ may extend over a set of choices not included in
  the training data, and hence may not generalize at all (or only
  poorly).  This can be verified by considering some extreme cases.
  Suppose that there are two actions $a$ and $b$ with action $a$
  occurring $10^6$ times and action $b$ occuring $10^2$ times.  Since
  action $b$ occurs only a $10^{-4}$ fraction of the time, a learning
  algorithm forced to trade off between predicting the expected value
  of $r_a$ and $r_b$ overwhelmingly prefers to estimate $r_a$ well at
  the expense of accurate estimation for $r_b$.  And yet, in
  application, action $b$ may be chosen by the argmax.  This problem
  is only worse when action $b$ occurs zero times, as might commonly
  occur in exploration situations.
\item {\it Bandit approaches}.  In the standard setting these
  approaches suffer from the curse of dimensionality, because they
  must be applied conditioned on $X$.  In particular, applying them
  requires data linear in $X \times A$, which is extraordinarily
  wasteful.  In essence, this is a failure to take advantage of
  generalization.
\item {\it Contextual Bandits}.  Existing approaches to contextual
  bandits such as EXP4 \cite{auer02b} or Epoch Greedy
  \cite{langford08b}, require either interaction to gather data or
  require knowledge of the probability the logging policy chose the
  action $a$.  In our case the probability is unknown, and it may in
  fact always be $1$.
\item {\it Exploration Scavenging}. It is possible to recover
  exploration information from action visitation frequency when a
  logging policy chooses actions independent of the input $x$ (but
  possibly dependent on history)~\cite{langford08}.  This doesn't fit
  our setting, where the logging policy is surely dependent on the
  query.
\end{enumerate}

\subsection*{Our Approach}

Our approach 
naturally breaks down into three steps.
\begin{enumerate}
\item For each event $(x,a,r_a)$, estimate the probability
  $\hat{\pi}(a|x)$ that the logging policy chooses action $a$ using
  regression.
\item For each event, create a synthetic controlled contextual bandit
  event according to $(x, a, r_a, 1/\max\{\hat{\pi}(a|x),\tau \})$ where $\tau >0$ is
  some parameter.  The fourth element in this tuple, $1/\max\{\hat{\pi}(a|x),\tau \}$, is 
  an {\it importance weight} that specifies how important the current event is for training.
  The parameter $\tau$ may appear mysterious at first,
  but is critical for numeric stability.
\item Apply an offline contextual bandit algorithm to the set of
  synthetic contextual bandit events.  In our experimental results a
  variant of the argmax regressor is used with two critical modifications:
\begin{enumerate} 
\item We limit the scope of the argmax to those actions with positive
  probability.
\item We importance weight events so that the training process
  emphasizes good estimation for each action equally.
\end{enumerate}
  It should be emphasized that the theoretical analysis in this paper
  applies to \emph{any} algorithm for learning on contextual bandit
  events---we chose this one because it is a simple modification on
  existing (but fundamentally broken) approaches.
\end{enumerate}

Three critical questions arise when considering this approach.
\begin{enumerate}
\item What does $\hat{\pi}(a|x)$ mean, given that the logging policy may
  be deterministically choosing an action (ad) $a$ given features $x$?  The
  essential observation is that a policy which deterministically
  chooses action $a$ on day $1$ and then deterministically chooses action $b$
  on day $2$ can be treated as randomizing between actions $a$ and $b$
  with probability $0.5$ when the number of events is the same each
  day, and the events are IID.  Thus $\hat{\pi}(a|x)$ is an estimate of
  the expected frequency with which action $a$ would be displayed given
  features $x$ over the timespan of the logged events.  In
  section~\ref{s:analysis} we show that this approach is sound in the
  sense that in expectation it provides an unbiased estimate of the
  value of new policy.  
\item How do the inevitable errors in $\hat{\pi}(a|x)$ influence the
  process?  It turns out they have an effect which is dependent on
  $\tau$.  For very small values of $\tau$, the estimates of
  $\hat{\pi}(a|x)$ must be extremely accurate to yield good performance
  while for larger values of $\tau$ less accuracy is required.  In
  Section~\ref{s:stoc}, we prove this robustness property.
\item What influence does the parameter $\tau$ have on the final
  result?  While creating a bias in the estimation process, it turns
  out that the form of this bias is mild and relatively
  reasonable---actions which are displayed with low frequency conditioned
  on $x$ effectively have an underestimated value.  This is exactly as
  expected for the limit where actions have \emph{no} frequency.
  In section~\ref{s:stoc} we prove this.
\end{enumerate}

We close with a generalization from policy evaluation to policy
selection with a sample complexity bound in section~\ref{s:opt} and
then experimental results in section~\ref{s:experiments} using a real
ad dataset.

\section{Formal Problem Setup and Assumptions}
\label{s:setup}

Let $\pi_1,...,\pi_T$ be $T$ policies, where, for each $t$, $\pi_t$
is a function mapping an input from $X$ to a (possibly deterministic)
distribution over $A$.  The learning algorithm is given a
dataset of $T$ samples, each of the form $(x,a,r_{a}) \in X \times A
\times [0,1]$, where $(x,r)$ is drawn from $D$ as described in
Section~\ref{s:bandit}, and the action $a \sim \pi_t(x)$ is chosen
according to the $t$th policy.  We denote this random process by
$(x,a,r_{a}) \sim (D,\pi_t(\cdot|x))$.  Similarly, interaction with the $T$ policies results in
a sequence $S$ of $T$ samples, which we denote $S\sim\Qexpl$.  The learner is not given prior knowledge of the $\pi_t$.

\subsection*{Offline policy estimator}
Given a dataset of the form
\begin{equation}
S=\{(x_t,a_t,r_{t,a_t})\}_{t=1}^T,
\end{equation}
where $\forall t,  x_t \in X, a_t \in A, r_{t,a_t} \in [0,1]$,
we form a predictor $\hat\pi:X\times A \to [0,1]$ and then use
it with a threshold $\tau \in [0,1]$ to form an offline estimator for
the value of a policy $h$.

Formally, given a new policy $h:X\to A$ and a dataset $S$, define the estimator:
\begin{equation}
\label{e:est}
\hat{V}^h_{\hat\pi}(S) = \frac{1}{|S|} \sum_{(x,a,r)\in S} \frac{r_a I(h(x)=a)}{\max\{\hat\pi(a|x),\tau\}},
\end{equation}
where $I(\cdot)$ denotes the indicator function.

The purpose of $\tau$ is to upper bound the individual terms in the sum and is similar to previous methods ~\cite{Owen98safeand}.

\section{Theoretical Results}
\label{s:analysis}

We now present our algorithm and main theoretical results. The main
idea is twofold: first, we have a policy estimation step, where we
estimate the (unknown) logging policy (analyzed in
Subsection~\ref{s:stoc}); second, we have a policy optimization step,
where our we utilize our estimated logging policy (analyzed in
Subsection~\ref{s:opt}).  Our main result, Theorem~\ref{t:opt},
provides a generalization bound --- addressing the issue of how both
the estimation and optimization error contribute to the total error.


The logging policy $\pi_t$ may be deterministic, implying that
conventional approaches relying on randomization in the logging policy
are not applicable.  We show next that this is ok when the world is
IID and the policy varies over its actions.  We effectively
substitute the standard approach of randomization in the algorithm for
randomization in the world.

A basic claim is that the estimator is expectation equivalent to a
stochastic policy defined as follows:
\begin{equation}
\label{d:stoc_p}
\pi(a|x) = \EE_{t \sim \mathrm{UNIF}(1,\ldots,T)} [ \pi_t(a|x) ],
\end{equation}
where $\mathrm{UNIF}(\cdots)$ denotes the uniform distribution.
The stochastic policy $\pi$ chooses an action uniformly at random over the $T$ policies $\pi_t$.  Our first result is that the expected value of our estimator is the same when the world chooses actions according to either $\pi$ or to the sequence of policies $\pi_t$.  Although this result and its proof are straight-forward, it forms the basis for the rest of the results in our paper.  Note that the policies $\pi_t$ may be arbitrary but we have assumed that they do not depend on the data used for evaluation.  Allowing for the offline evaluation of policies using the same data they are trained on is an important open problem.

\begin{theorem}
\label{t:main}
For any contextual bandit problem $D$ with identical draws over $T$
rounds, for any sequence of possibly stochastic policies $\pi_t(a|x)$
with $\pi$ derived as above, and for any predictor $\hat{\pi}$,
\begin{equation}
\label{fund}
E_{S \sim \Qexpl}\hat{V}^h_{\hat{\pi}}(S) = E_{(x,\vec{r})\sim D, a \sim \pi(\cdot|x)} \frac{r_aI(h(x)=a)}{\max\{ \hat{\pi}(a|x),\tau \}}
\end{equation}
\end{theorem}
This theorem relates the expected value of our estimator when $T$
policies are used to the much simpler and more standard setting where
a single fixed stochastic policy is used.  
\begin{proof}
\begin{eqnarray*}
& & E_{(x,\vec{r})\sim D, a \sim \pi(\cdot|x)} \frac{r_aI(h(x)=a)}{\max\{ \hat{\pi}(a|x),\tau \}} \\
& = & E_{(x,\vec{r})\sim D} \sum_a \pi(a|x) \frac{r_aI(h(x)=a)}{\max\{ \hat{\pi}(a|x),\tau \}} \\
& = & E_{(x,\vec{r})\sim D} \sum_a \frac{1}{T} \sum_t \pi_t(a|x) \frac{r_aI(h(x)=a)}{\max\{ \hat{\pi}(a|x),\tau \}} \\
& = & E_{(x,\vec{r})\sim D} \frac{1}{T} \sum_t \sum_a  \pi_t(a|x) \frac{r_aI(h(x)=a)}{\max\{ \hat{\pi}(a|x),\tau \}} \\
& = & E_{(x,\vec{r})\sim D} \frac{1}{T} \sum_t E_{a \sim \pi_t(\cdot|x)} \frac{r_aI(h(x)=a)}{\max\{ \hat{\pi}(a|x),\tau \}} \\
& = & E_{(x,\vec{r})^T\sim D^T} \frac{1}{T} \sum_t E_{a_t \sim \pi_t(\cdot|x_t)} \frac{r_{i,a_t}I(h(x_t)=a_t)}{\max\{ \hat{\pi}(a_t|x_t),\tau \}} \\
& = & E_{(x,\vec{r})^T\sim D^T, a_t \sim \pi_t(\cdot|x)} \frac{1}{T} \sum_t \frac{r_{i,a_t}I(h(x_t)=a_t)}{\max\{ \hat{\pi}(a_t|x_t),\tau \}} \\
& = & E_{S\sim \Qexpl} \frac{1}{|S|} \sum_{(x,a,r)\in S} \frac{r_a I(h(x)=a)}{\max\{ \hat{\pi}(a|x),\tau \}}
\end{eqnarray*}
Each equality follows form linearity of expectation, relabeling, or
the definition of expectation.  The identical draws assumption is used
in 6th equality.
\end{proof}

\subsection{Policy Estimation}
\label{s:stoc}

In this section we show that for a suitable choice of $\tau$ and
$\hat{\pi}$ our estimator is sufficiently accurate for evaluating new
policies $h$.  We aggressively use the simplification of the previous
section, which shows that we can think of the data as generated by a fixed
stochastic policy $\pi$, i.e. $\pi_t = \pi$ for all $t$.

\newcommand{\reg}{\mathop{\rm reg}}

For a given estimate $\hat{\pi}$ of $\pi$ define the ``regret'' to be
a function $\reg\!:\!X\to[0,1]$ by
\begin{equation}
\label{e:squared_regret}
\reg(x) = \max_{a \in \mathcal{A}} \left[{ (\pi(a|x) - \hat\pi(a|x))^2 }\right].
\end{equation}

Our first result is that the new estimator is consistent.  In the following theorem statement, $I(\cdot)$ denotes the indicator function, $\pi(a|x)$ the probability that the logging policy chooses action $a$ on input $x$, and $\hat{V}_{\hat\pi}^h$ our estimator as defined by Equation~\ref{e:est} based on parameter $\tau$.

\begin{lemma}
\label{l:res_est}
Let $\hat\pi$ be any function from $X$ to distributions over actions $A$.  Let $h:X \to A$ be any deterministic policy.  Let $V^h(x) = \EE_{r\sim D(\cdot|x)}[r_{h(x)}]$ denote the expected value of executing policy $h$ on input $x$.  We have that 
$$
\EE_x \left[{ I(\pi(h(x)|x)\geq\tau)\cdot\left( V^h(x) - \frac{\sqrt{\reg(x)}}{\tau}\right) }\right] ~\leq~
$$
$$
\EE[\hat{V}_{\hat\pi}^h] ~\leq~
$$
$$
 V^h + \EE_x \left[{ I(\pi(h(x)|x)\geq\tau) \cdot \frac{\sqrt{\reg(x)}}{\tau}}\right].
$$
In the above, the expectation $\EE[\hat{V}_{\hat\pi}^h]$ is taken over all sequences of $T$ tuples $(x,a,r)$ where $(x,r) \sim D$ and $a \sim \pi(\cdot|x)$.\footnote{Note that varying $T$ does not change the expectation of our estimator, so $T$ has no effect in the theorem.}
\end{lemma}

This lemma bounds the bias in our estimate of $V^h(x)$.  There are two
sources of bias---one from the error of $\hat{\pi}(a|x)$ in estimating
$\pi(a|x)$, and the other from threshold $\tau$.  For the first
source, it's crucial that we analyze the result in terms of the
squared loss rather than (say) $\ell_\infty$ loss, as reasonable
sample complexity bounds on the regret of squared loss estimates are
achievable.  
\begin{proof}
Consider a fixed $x$.  Define the following quantity
\begin{eqnarray*}
\delta_x &=& \frac{\pi(h(x)|x)}{\max\{\hat \pi (h(x)|x),\tau\}}V^h(x) - V^h(x).
\end{eqnarray*}
The quantity $\delta_x$ is the error of our estimator conditioned on $x$ and satisfies $\EE_x[\delta_x] = \EE[\hat{V}_{\hat\pi}^h] - V^h$.
Note that $|\delta_x| \leq \left|{ \frac{\pi(h(x)|x)}{\max\{\hat \pi (h(x)|x),\tau\}} - 1}\right|$.

We consider two disjoint cases.  

First, suppose  that $\pi(h(x)|x) < \tau$.  Then, $\delta_x$ is less than or equal to zero, due to the max operation in the denominator and the fact that rewards are positive.  Thus, we have that $E[\hat{V}_{\hat{\pi}}^h] \leq V^h $, when the expectation is taken over the $x$ for which $\pi(h(x)|x) < \tau$.  As an aside, note that $|\delta_x|$ can have magnitude as large as 1.  In other words, in this situation, the estimator may drastically underestimate the value of policy $h$ but will never overestimate it.

Second, suppose that $\pi(h(x)|x) \geq \tau$.  
Then, we have that 
\begin{eqnarray*}
\lefteqn{ |\delta_x| } \\
& \leq & \left |{\frac{\pi(h(x)|x)-\max\{\hat \pi (h(x)|x),\tau\}}{\max\{\hat \pi (h(x)|x),\tau\}}} \right|\\
& \leq & \frac{ \sqrt{\reg(x)}}{\tau}.
\end{eqnarray*}

Expanding $\delta_x$ and taking the expectation over $x$ for which $\pi(h(x)|x) \geq \tau$ yields the desired result.
\end{proof}

\begin{corollary}
\label{c:res_est}
Let $\hat\pi$ be any function from $X$ to distributions over actions $A$.  Let $h:X \to A$ be any deterministic policy.  If $\pi(h(x)|x) \geq \tau$ for all inputs $x$, then
\begin{equation}
\label{e:est_bound_tau}
|\EE[\hat{V}_{\hat\pi}^h]-V^h| ~\leq~
\frac{\sqrt{\EE_x[\reg(x)]}}{\tau}.
\end{equation}
\end{corollary}
\begin{proof}
Follows from examining the second part of the proof of Lemma~\ref{l:res_est} and applying Jensen's inequality.
\end{proof}

Lemma~\ref{l:res_est} shows that the expected value of our estimate
$\hat{V}_\pi^h$ of a policy $h$ is an approximation to a lower bound of the
true value of the policy $h$ where the approximation is due to errors
in the estimate $\hat{\pi}$ and the lower bound is due to the
threshold $\tau$.  When $\hat{\pi} = \pi$, then the statement of Lemma~\ref{l:res_est} simplifies to
$$
\EE_x \left[{ I(\pi(h(x)|x)\geq\tau)\cdot V^h(x) }\right] ~\leq~
\EE[\hat{V}_{\hat\pi}^h] ~\leq~
 V^h.
$$
Thus, with a perfect predictor of $\pi$, the expected value of the estimator $\hat{V}_{\hat\pi}^h$ is a guaranteed lower bound on the true value of policy $h$.  However, as the left-hand-side of this statement suggests, it may be a very loose bound, especially if the action chosen by $h$ often has a small probability of being chosen by $\pi$.

The dependence on $1/\tau$ in Lemma~\ref{l:res_est} is somewhat
unsettling, but unavoidable.  Consider an instance of the
bandit problem with a single input $x$ and two actions $a_1,a_2$.
Suppose that $\pi(a_1|x) = \tau + \epsilon$ for some positive
$\epsilon$ and $h(x) = a_1$ is the policy we
are evaluating.  Suppose further that the rewards are always $1$ and
that $\hat{\pi}(a_1|x) = \tau$.  Then, the estimator satisfies
$E[\hat{V}_{\hat{\pi}}^h] = \pi(a_1|x)/{\hat{\pi}(a_1|x)} = (\tau +
\epsilon)/\tau$.  Thus, the expected error in the estimate is
$E[\hat{V}_{\hat{\pi}}^h] - V^h = |(\tau + \epsilon)/\tau - 1| =
\epsilon/\tau$, while the regret of $\hat{\pi}$ is $(\pi(a_1|x) -
\hat{\pi}(a_1|x))^2 = \epsilon^2$.

\subsection{Policy Optimization}
\label{s:opt}

The previous section proves that we can effectively evaluate a policy
$h$ by observing a stochastic policy $\pi$, as long as the actions
chosen by $h$ have adequate support under $\pi$, specifically
$\pi(h(x)|x) \geq \tau$ for all inputs $x$.  However, we are often
interested in choosing the best policy $h$ from a set of policies
$\cal H$ after observing logged data.  Furthermore, as described in
Section~\ref{s:setup}, the logged data are generated from $T$ fixed,
possibly deterministic, policies $\pi_1, \ldots, \pi_T$ as described
in section~\ref{s:setup} rather than a single stochastic policy.  As
in Section~\ref{s:analysis} we define the stochastic policy $\pi$ by
Equation~\ref{d:stoc_p},
\begin{equation*}
\pi(a|x) = \EE_{t \sim \mathrm{UNIF}(1,\ldots,T)} [ \pi_t(a|x) ]
\end{equation*}
The results of Section~\ref{s:stoc} apply to the policy optimization
problem.  However, note that the data are now assumed to be drawn from
the execution of a sequence of $T$ policies $\pi_1,\ldots, \pi_T$,
rather than by $T$ draws from $\pi$.

Next, we show that it is possible to compete well with the best
hypothesis in $\cal H$ that has adequate support under $\pi$ (even
though the data are not generated from $\pi$).  


\begin{theorem}
\label{t:opt}

Let $\hat\pi$ be any function from $X$ to distributions over actions $A$.
Let $\cal H$ be any set of deterministic policies.
Define $\tilde{\cal H} = \{ h\in\mathcal{H}~|~\pi(h(x)|x) > \tau, ~\forall ~x \in X \}$ %
and $\tilde{h} = \argmax_{h \in \tilde{ \mathcal{H} }} \{ V^h \}$.
Let $\hat{h} = \argmax_{h \in {\mathcal{H}}} \{ \hat{V}_{\hat\pi}^{h} \}$ be the hypothesis that maximizes the empirical value estimator defined in Equation~\ref{e:est}.  Then,
with probability at least $1-\delta$,
\begin{equation}
V^{\hat{h}} \geq V^{\tilde{h}} - \frac{2}{\tau}\left( \sqrt{\EE_x[\reg(x)]} + \sqrt{\frac{\ln(2|H|/\delta)}{2T}} \right),
\end{equation}
where $\reg(x)$ is defined, with respect to $\pi$, in Equation~\ref{e:squared_regret}.
\end{theorem}
\begin{proof}
First, given a dataset $(x_t,a_t,r_{t,a_t})$, $t= 1,\ldots,T$, generated by the process described in Section~\ref{s:setup}, note that it is straight-forward to apply Hoeffding's bound \cite{Hoeffding63} to the random variables $X_t = \frac{I(h(x_t)=a_t)r_{t,a_t}}{\max\{\hat{\pi}(a_t|x_t),\tau\}}$, to show that 
$|\hat{V}_{\hat\pi}^{h} - \EE[\hat{V}_{\hat\pi}^{h}]| \leq \frac{1}{\tau}\sqrt{\frac{\ln(2/\delta)}{2T}}$ 
holds with probability at least $1-\delta$, for a fixed policy $h$.  It is important to note here that the $X_t$ are independent but not identical, since the action at time $t$ is chosen according to policy $\pi_t$.  The previous argument can be made to hold for all $h \in H$ by replacing $\delta$ with $\delta/|H|$ and applying the union bound.

Let $Q = \Qexpl$ be the distribution over sequences of $T$ samples $(x,a,r_a) \in X \times A \times [0,1]$ generated by executing the $T$ logging policies $\pi_t$ in sequence, as described in section~\ref{s:setup}.  Let $Q' = \Qpexpl$ be the distribution over samples of the form $(x,a,r_a) \in X \times A \times [0,1]$ such that $(x,r) \sim D$ and $a \sim \pi(\cdot|x)$.  The $T$ samples used in the estimator $\hat{V}_{\pi}^{{h}}$ are obtained from a single draw from $Q$. 

\def\vhathat{\hat{V}_{\hat\pi}^{\hat{h}}}
\def\vhattilde{\hat{V}_{\hat\pi}^{\tilde{h}}}
\def\epssl{\epsilon_{\mathrm{SL}}}

Now, we have that
\begin{eqnarray*}
\lefteqn{ V^{\hat{h}} } \\
& \geq & \EE_{Q'}{[\vhathat]} - \EE_x \left[{ I(\pi(\hat{h}(x)|x)\geq\tau) \cdot \frac{\sqrt{\reg(x)}}{\tau}}\right] \\
& \geq & \EE_{Q'}{[\vhathat]} - \EE_x \left[{\frac{\sqrt{\reg(x)}}{\tau}}\right] \\
& \geq & \EE_{Q'}{[\vhathat]} - \frac{ \sqrt{\EE_x[\reg(x)]} }{\tau} \\
& = & \EE_{Q}{[\vhathat]} - \frac{ \sqrt{\EE_x[\reg(x)]} }{\tau} \\
& \geq & \vhathat - \frac{\sqrt{\EE_x[\reg(x)]}}{\tau} - \frac{1}{\tau}\sqrt{\frac{\ln(2|H|/\delta)}{2T}} \\
& \geq & \vhattilde - \frac{\sqrt{\EE_x[\reg(x)]}}{\tau} - \frac{1}{\tau}\sqrt{\frac{\ln(2|H|/\delta)}{2T}} \\
& \geq & \EE_{Q}[\vhattilde] - \frac{\sqrt{\EE_x[\reg(x)]}}{\tau} - \frac{2}{\tau}\sqrt{\frac{\ln(2|H|/\delta)}{2T}} \\
& = & \EE_{Q'}[\vhattilde] - \frac{\sqrt{\EE_x[\reg(x)]}}{\tau} - \frac{2}{\tau}\sqrt{\frac{\ln(2|H|/\delta)}{2T}} \\
& \geq & V^{\tilde{h}} - \frac{2\sqrt{\EE_x[\reg(x)]}}{\tau} - \frac{2}{\tau}\sqrt{\frac{\ln(2|H|/\delta)}{2T}}.
\end{eqnarray*}

The first step follows from Lemma~\ref{l:res_est}.  The second from the fact that regret is always non-negative.  The third from an application of Jensen's inequality. The forth and eighth from an application of Theorem~\ref{t:main}.  The fifth and seventh from an application of Hoeffding's bound as detailed above.  The sixth from the definition of $\tilde{h}$.  The final step follows from Corollary~\ref{c:res_est} and observing that $\tilde{h} \in \tilde{H}$.
\end{proof}

The proof of Theorem~\ref{t:opt} relies on the
lower-bound property of our estimator (the left-hand side of
Inequality stated in Lemma~\ref{l:res_est}).  In other words, if $\cal H$ contains a very good
policy that has little support under $\pi$, we will not be able to
detect that by our estimator.  On the other hand, our estimation is
safe in the sense that we will never drastically overestimate the
value of any policy in $\cal H$.  This ``underestimate, but don't
overestimate'' property is critical to the application of optimization
techniques, as it implies we can use an unrestrained learning
algorithm to derive a warm start policy.

\section{Empirical Evaluation}
\label{s:experiments}

We evaluated our method on a real-world {\it Internet advertising}
dataset.  We have obtained proprietary data from an online advertising
company, covering a period of approximately one month. The data are
comprised of logs of events $(x,a,y)$, where each event represents a
visit by a user to a particular web page $x$, from a set of web pages
$X$.  From a large set of advertisements $A$, the commercial system
chooses a single ad $a$ for the topmost, or most prominent position.
It also chooses additional ads to display, but these were ignored in
our test.  The output $y$ is an indicator of whether the user clicked
on the ad or not.

The total number of ads in the data set is approximately $880,000$.
The training data consist of $35$ million events.  The test data
contain $19$ million events occurring after the events in the training data.
The total number of distinct web pages is approximately $3.4$ million.

We trained a policy $h$ to choose an ad, based on the current page, to
maximize the probability of click.  For the purposes of learning, each
ad and page was represented internally as a sparse high-dimensional
feature vector.  The features correspond to the words that appear in
the page or ad, weighted by the frequency with which they appear.  Each
ad contains, on average, $30$ ad features and each page, approximately
$50$ page features.  The particular form of $f$ was linear over all features of its input $(x,y)$, which is a sparse high-dimensional feature vector representing the combination of the page and ad.\footnote{Technically the feature vector that the regressor uses is the Cartesian product of the page and ad vectors.}
For instance, every pair of possible words had a corresponding feature.  For example, given the two words ``apple'' and ``ipod'', the corresponding feature ``apple-ipod'' has a value of $0.25$ when the first word, ``apple'', appeared in the page $x$ with frequency 0.5 and the second word, ``ipod'', appeared in the ad $a$ with frequency 0.5.

Using \emph{all} the data, we modeled the logging policy using simple
empirical estimation:
\begin{equation}
\label{e:prob_ml}
\hat{\pi}(a|x) = \frac{|\{ t | (a_t = a) \wedge (x_t = x) \} |}{|\{ t | x_t = x \} |}.
\end{equation}
In words, for each page and ad pair $(x,a)$, we computed the number of times $a$ appeared on page $x$ in the data.
The decision to use all of the data requires careful consideration.
Some alternatives to consider are:
\begin{enumerate}
\item Training data only.  Since the set of ads changes over time,
  many ads appearing in the test data do not occur at all in the
  training data.  Consequently, reliably predicting the performance on
  test data is problematic.
\item Training data for training set and test data for test set.  This
  approach has an inherent bias towards incorrectly high scores on the
  test set.  In an extreme case, suppose that only one ad appears on a
  (rare) webpage in the test set.  Then, any policy selecting from
  amongst the set of appearing ads must select this ad.
\item All data.  This approach means that policies must generally
  select from a larger set of ads than are available at any moment in
  time for the live system, implying that the policy evaluation is
  generally pessimistic.  Note that the logging policy in contrast is
  \emph{optimistically} evaluated, because the set of test-time
  available ads is smaller than the set of ads available over both
  test-time and train-time ads, implying the frequency estimates for
  test-time ads on the train+test dataset are generally smaller than
  an estimate using just test-time ads.\footnote{As an extreme
    example, suppose we log data for two days and we use the first day
    for training and the second day for testing.  Suppose that only a
    single ad $a_1$ appears in the train set, and a single ad $a_2$
    appears in the test set, due to the fact that the budget for ad
    $a_1$ ran out after the first day.  Our empirical estimate of
    $\hat{\pi}(a_2 | x)$ on the test set used in the denominator of
    our estimator (Equation~\ref{e:prob_ml}) will be $1/2$.  In fact
    the true probability of $a_2$ on the test set is $1$.  Thus, the
    value of the logging policy will be over estimated by a factor of
    2.  Suppose further that ad $a_1$ is indeed better than $a_2$.
    The evaluation of a policy that always chooses the better ad,
    $a_1$, using Equation~\ref{e:prob_ml} will be zero, a drastic
    underestimate of its true value.}  These smaller-than-necessary
  frequency estimates imply that the logging policy evaluation is
  optimistic since events are weighted by the inverse frequency.
  Consequently, this choice provides a conservative estimate for new
  policies and an optimistic choice for the older (logging) policy.
\end{enumerate}

The particular policy that was optimized, had an argmax form: $h(x) =
\argmax_{a \in C(X)} \{f(x,a)\}$, with a crucial distinction from
previous approaches in how $f(x,a)$ was trained.  Here $f:X\times A
\to [0,1]$ is a regression function that is trained to estimate
probability of click, and $C(X) = \{a \in A~|~\hat{\pi}(a|x) > 0\}$ is
a set of feasible ads.

The training samples were of the form
$(x,a,y)$, where $y = 1$ if the ad $a$ was clicked after being shown
on page $x$ or $y = 0$ if it wasn't clicked.  The regressor $f$ was
chosen to approximately minimize the \emph{weighted} squared loss:
$\frac{(y - f(x,a))^2}{{\max\{\hat\pi(a_t|x_t),\tau\}}}$.  

Stochastic gradient descent was used to minimize the squared loss on the training
data.

During the evaluation, we computed the estimator on the test data
$(x_t,a_t,y_t)$:
\begin{equation}
\label{e:test_eval}
\hat{V}^h_{\hat\pi} = \frac{1}{T}\sum_{t=1}^{T} \frac{y_t I(h(x_t)=a_t)}{\max\{\hat\pi(a_t|x_t),\tau\}}.
\end{equation}
 As mentioned in the introduction, this estimator is biased due to
the use of the parameter $\tau > 0$.  As shown in the analysis of
Section~\ref{s:analysis}, this bias typically results in an
underestimate of the true value of the policy $h$.  

\chomp{Our goal of determining the advantage of using policy $h$ over that of
the logging policy can be achieved by estimating the value of the
logging policy $\pi$ with the same bias effect:
\begin{equation}
\label{e:base_policy_eval}
\hat{V}^{\hat\pi}_{\hat\pi} = \frac{1}{T}\sum_{t=1}^{T} \frac{y_t \hat{\pi}(a_t|x_t)}{\max\{\hat\pi(a_t|x_t),\tau\}}.
\end{equation}
This argument can be formalized as follows.  The form of our estimator as written in Equation~\ref{e:est} makes sense only for deterministic policies $h$.
It is easily modified to handle the evaluation of stochastic policies by replacing the indicator function $I(h(x) = a)$ with $h(a|x)$, the probability that $h$ will choose $a$.  Equation~\ref{e:base_policy_eval} is
therefore the result of evaluating the stochastic policy $\hat{\pi}$, our best estimation of the logging policy.}

We experimented with different thresholds $\tau$ and parameters of our
learning algorithm.\footnote{For stochastic gradient descent, we
  varied the learning rate over $5$ fixed numbers ($0.2, 0.1, 0.05,
  0.02, 0.01$) using 1 pass over the data.  We report on the test
  results for the value with the best training error.}

\subsection{Results}

\begin{center}
\label{t:table}
\begin{small}
\begin{tabular}{|c||c|c|c|}
\hline
Method & $\tau$ & Estimate & Interval\\
\chomp{\hline
Base & 0.01 & 0.0200 & [0.0194,0.0214] \\}
\hline
Learned & 0.01 & 0.0193 & [0.0187,0.0206] \\
\hline
Random & 0.01 & 0.0154 & [0.0149,0.0166] \\
\chomp{\hline
Base & 0.05 & 0.0177 & [0.0174,0.0183]\\}
\hline
Learned & 0.05 & 0.0132 & [0.0129,0.0137] \\
\hline
Random & 0.05 & 0.0111 & [0.0109,0.0116] \\
\hline
Naive & 0.05 & 0.0 & [0,0.0071] \\
\hline
\end{tabular}
\end{small}
\end{center}
The Interval column is computed using the relative entropy form
of the Chernoff bound with $\delta = 0.05$ which holds under the
assumption that variables, in our case the samples used in the computation of the estimator (Equation~\ref{e:test_eval}), are IID.
Note that this computation is
slightly complicated because the range of the variables is
$[0,1/\tau]$ rather than $[0,1]$ as is typical.  This is handled by
rescaling by $\tau$, applying the bound, and then rescaling the
results by $1/\tau$.

\chomp{The``Base'' policy is the same in the two entries in the table
above and its evaluation is computed according to Equation~\ref{e:base_policy_eval}.  The base policy does not depend on $\tau$, but $\tau$ does
influence the estimate.}

The ``Random'' policy is the policy that chooses randomly from the set of feasible ads:
$\mathrm{Random}(x) = a \sim \mathrm{UNIF}(C(X))$,
where $\mathrm{UNIF}(\cdot)$ denotes the uniform distribution.

The ``Naive'' policy corresponds to the theoretically flawed
supervised learning approach detailed in the introduction.  The
evaluation of this policy is quite expensive, requiring one evaluation
per ad per example, so the size of the test set is reduced to $8373$
examples with a click, which reduces the significance of the results.
We bias the results towards the naive policy by choosing the
chronologically first events in the test set (i.e. the events most
similar to those in the training set).  Nevertheless, the naive policy
receives $0$ reward, which is significantly less than all other
approaches.  A possible fear with the evaluation here is that the
naive policy is always finding good ads that simply weren't explored.
A quick check shows that this is not correct--the naive argmax simply
makes implausible choices.  Note that we report only evaluation
against $\tau = 0.05$, as the evaluation against $\tau = 0.01$ is not
significant, although the reward obviously remains $0$.

The ``Learned'' policies do depend on $\tau$.  As suggested by
Theorem~\ref{t:opt}, as $\tau$ is decreased, the effective set of
hypotheses we compete with is increased, thus allowing for better
performance of the learned policy.  Indeed, the estimates for both the
learned policy and the random policy improve when we decrease $\tau$
from $0.05$ to $0.01$.  

The empirical click-through rate on the test set was $0.0213$, which
is slightly larger than the estimate for the best learned policy.
However, this number is not directly comparable since the estimator
provides a lower bound on the true value of the policy due to the bias
introduced by a nonzero $\tau$ and because any deployed policy chooses
from only the set of ads which are available to display rather than
the set of all ads which might have been displayable at other points
in time.

The empirical results are generally consistent with the theoretical
approach outlined here---they provide a consistently pessimal estimate
of policy value which nevertheless has sufficient dynamic range to
distinguish learned policies from random policies, learned policies
over larger spaces (smaller $\tau$) from smaller spaces (larger
$\tau$), and the theoretically unsound naive approach from
sounder approaches which choose amongst the the explored space of ads.

\section{Conclusion}
\label{s:discussion}

We stated, justified, and evaluated theoretically and empirically the
first method for solving the warm start problem for exploration from
logged data with controlled bias and estimation.  This problem is of
obvious interest to applications for internet companies that recommend
content (such as ads, search results, news stories, etc...) to users.

However, we believe this also may be of interest for other application
domains within machine learning.  For example, in reinforcement
learning, the standard approach to offline policy evaluation is based
on importance weighted samples~\cite{Ng00,D00}.  The basic results
stated here could be applied to RL settings, eliminating the need to
know the probability of a chosen action explicitly, allowing an RL
agent to learn from external observations of other agents.

The main restrictive assumption adopted by the Exploration Scavenging
paper~\cite{langford08} is that the logging policy chooses actions
independently of the input.  We have introduced a new method that
works when this assumption is violated.  On the other hand, we have
required the logging policy be a sequence of fixed, possibly
deterministic, policies, whereas the Exploration Scavenging paper
allowed for the use of logging policies that learn and adapt over
time.  An interesting situation occurs when you allow $\pi_t$ to
depend on the history up to time $t$.  In this setting the policy may
both adapt (like in the Exploration Scavenging paper) and choose
actions dependent on the current input.  Is there an offline policy
estimator which can work in this setting?  The most generic answer is
no, but there may exist some natural constraint which encapsulates the
approach discussed here, as well as in the earlier paper.

{\small
\bibliography{alstrehl}
\bibliographystyle{icml2010}
}

\end{document}